\documentclass[]{interact}

\usepackage{epstopdf}
\usepackage[caption=false]{subfig}

\usepackage[natbibapa,nodoi]{apacite} 
\theoremstyle{plain}
\newtheorem{theorem}{Theorem}[section]
\newtheorem{lemma}[theorem]{Lemma}
\newtheorem{corollary}[theorem]{Corollary}

\newtheorem{property}[theorem]{Property}

\theoremstyle{definition}
\newtheorem{definition}[theorem]{Definition}

\theoremstyle{remark}

\usepackage{tabularx}
\usepackage{xcolor}
\usepackage{etoolbox} 
\usepackage{graphicx}
\usepackage{hyperref}

\usepackage[T1]{fontenc}
\usepackage{amssymb,amsmath}
\usepackage{color}
\usepackage{graphicx}

\newtoggle{anonymous}
\togglefalse{anonymous}

\newcommand{\leqcup}{\rotatebox[origin = cc]{-90}{$\geq$}}                                                          

\begin{document}


\title{Predictability maximization and the origins of word order harmony}

\iftoggle{anonymous}{}
{ 
  \author{
  \name{Ramon Ferrer-i-Cancho\thanks{CONTACT Ramon Ferrer-i-Cancho. Email: rferrericancho@cs.upc.edu}}
  \affil{Quantitative, Mathematical and Computational Linguistics Research Group, Departament de Ci\`encies de la Computaci\'o, Universitat Polit\`ecnica de Catalunya (UPC), Barcelona, Spain. ORCiD: 0000-0002-7820-923X}
  }
}

\maketitle

\begin{abstract}
We address the linguistic problem of the sequential arrangement of a head and its dependents from an information theoretic perspective. In particular, we consider the optimal placement of a head that maximizes the predictability of the sequence. We assume that dependents are statistically independent given a head, in line with the open-choice principle and the core assumptions of dependency grammar. We demonstrate the optimality of harmonic order, i.e., placing the head last maximizes the predictability of the head whereas placing the head first maximizes the predictability of dependents. We also show that postponing the head is the optimal strategy to maximize its predictability while bringing it forward is the optimal strategy to maximize the predictability of dependents. We unravel the advantages of the strategy of maximizing the predictability of the head over maximizing the predictability of dependents. Our findings shed light on the placements of the head adopted by real languages or emerging in different kinds of experiments.
\end{abstract}

\begin{keywords}
predictability maximization; optimal head placement; open-choice principle; word order harmony
\end{keywords}

\section{Introduction}


 
Languages exhibit a tendency to put heads before or after their dependents across distinct scales in the sentence. 
Orders where the head is placed first or last with respect to their dependents are called harmonic \citep{Biberauer2013a,Dryer2018a}. Among the distinct notions of harmonicity, here we adopt the basic notion that does not take into account how dependents are ordered within the phrase, across kinds of phrases and levels of organization in the hierarchy of the syntactic structure of a sentence \citep{Biberauer2013a}.

On top of the clause, consider the case of the ordering of the verb (i.e. the head) and subject and the object (i.e. the complements). Various sources of evidence suggest a preference for placing the verb last:
\begin{enumerate}
\item
Let us consider the sequence formed by subject (S), object (O), and verb (V).
The most frequent dominant word order in world languages is SOV (head last) \citep{wals-84,Hammarstroem2016a}. 
The fact that the number of languages increases as the verb moves from the beginning to the end of the sequence (Table \ref{word_order_statistics_table}) suggests that postponing the verb (the head) is favoured for some reason. The effect is more prominent when the frequency of a dominant order is measured in families (Table \ref{word_order_statistics_table}), which is more statistically robust than counts in languages. \footnote{Counts in families control for linguistic relatedness and is a more balanced way of sampling over the diversity of world languages. } 
\item
Gesturing experiments \citep{Goldin-Meadow2008a,Langus2010a,Meir2010b,Hall2013a,Hall2014a,Futrell2015a} 
\footnote{In this context, it is often cited \citet{Gibson2013a} but this work only reports on OV versus VO order, not the triple formed by S, V and O. }
where a robust strong preference for an order consistent with SOV (head last) was found in signers even when their spoken language did not have SOV as the dominant word order. 
\item
SOV has been hypothesized to arise in early stages in the evolution of spoken languages \citep{Gell-Mann2011a,Newmeyer2000}. In addition, SOV is often the dominant order found in sign languages which are at the early stages of community-level conventionalisation \citep{Meir2010a, Sandler2005a}.
\item
{\em In silico} experiments with neural networks have shown that SOV (head last) is the word order that emerges when languages are selected to be more easily learned by networks predicting the next element in a sequence \citep{Reali2009a}. 
\end{enumerate}
At a lower level of the clause, consider the case of a noun phrase: 
\begin{enumerate}
\item
In world languages, nouns tend to be placed harmonically, namely either first or last in a noun phrase formed by noun, adjective, numeral and demonstrative 
\iftoggle{anonymous}{\citep{Ferrer2023b_fake}}{\citep{Ferrer2023b}.}
\item
Artificial language learning experiments have reported that harmonic noun phrase word orders are preferred over non-harmonic orders by English-speaking adults and children \citep{Culbertson2012a,Culbertson2015a,Culbertson2017a}. Crucially, the finding holds even for speakers whose native language is not harmonic \citep{Culbertson2020a,Martin2024a}.
\end{enumerate}

\begin{table}
\caption{\label{word_order_statistics_table} 
The frequency of the placement of the verb in the triple formed by subject, verb and object in world languages showing a dominant word order. Source indicates the publication where the absolute frequencies are borrowed from. Unit indicates the unit of measurement of frequency: number of languages having a certain order as dominant or the number families where a certain order is the majority dominant order.   
Order indicates the position of the verb in the triple: 1 for verb initial (VSO and VOS), 2 for central verb placements (SVO and OVS) and 3 for verb final orderings (SOV and OSV).
}
\begin{center}
\begin{tabular}{@{}llrrr}
Source & Unit & Order & Frequency & Percentage \\
\hline
\citet{wals-84} & languages & 1 &  111 & 10.5 \\
                   &           & 2 &  444 & 42.0 \\
                   &           & 3 &  501 & 47.4 \\
                   &           & Total & 1056 & \\
\citet{Hammarstroem2016a} & languages & 1 &  677 & 13.2 \\
                          &           & 2 & 2157 & 41.3 \\
                          &           & 3 & 2294 & 44.7 \\
                          &           & Total & 5128 & \\
\citet{Hammarstroem2016a} & families  & 1 &   42 & 12.4 \\
                          &           & 2 &   58 & 17.0 \\
                          &           & 3 &  240 & 70.6 \\
                          &           & Total &  340 & \\                                           
\end{tabular}
\end{center}
\end{table}


Here we aim to investigate the problem of the optimal placement of a head and its dependents according to predictability, namely, the ease of prediction of next element in a sequence 
\iftoggle{anonymous}{\citep{Kamide2008a, Ferrer2013f_fake, Onnis2022a}.}{\citep{Kamide2008a, Ferrer2013f, Onnis2022a}.} 
We will provide some information theoretic foundations for the study of the ordering of heads from a predictability maximization perspective. We will show that (a) harmonic orders and (b) postponing the head and eventually placing it last, as the statistics of word orders suggest (Table \ref{word_order_statistics_table}), is supported by a principle of maximum predictability.

The remainder of the article is organized as follows. Section \ref{predictability_section} presents the information theoretic framework. Assuming that dependents are statistically independent given a head, it shows that a harmonic placement of the head is optimal. In particular, placing the head first maximizes the predictability of the dependents that will follow whereas placing the head last maximizes the predictability of the head from the preceding dependents. It also shows that the head needs to be produced as soon as possible if the predictability of each of the pending dependents has to be maximized while it must be postponed to maximize its predictability. This section is written to accommodate readers of increasing levels of mathematical competence or mathematical aspirations. Mathematically naive readers should just read up to section \ref{subsec:conditional_independence} and the corollaries. Readers with some mathematical background or noble aspirations should be able to understand the statements of the theorems (not their proofs). The remainder of readers should be able to understand it all. 
Section \ref{discussion_section} discusses the advantages of head predictability maximization over dependent predictability maximization. It also discusses the implications of our theoretical results for the placements of the head adopted by real languages or emerging in different kinds of experiments. 


\section{Placement of the head that maximizes predictability}

\label{predictability_section}

For simplicity, our theoretical setting assumes a sequence containing a head $l$ and $n$ dependents, $m_1,...,m_i,...,m_n$. Our framework is flexible and the head and its dependents can be individual words or phrases. 
Thus our notion of dependent comprises either an individual word (the dependent word in a dependency grammar setting) or a constituent (a dependent word and all the words in its syntactic dependency subtree). \footnote{For simplicity, we assume projectivity so that constituents are continuous. } We use the term dependent as an umbrella for sophisticated terminology from theoretical syntax such as modifiers, complements or adjuncts. \footnote{We do not need such a complex terminology here. The purpose of such terminology is fine-grained description. Instead, our emphasis is on explanation and prediction by means of a parsimonious approach.} 
For instance, the English phrase ``a black cat'' has a head (the noun ``cat'') and two dependents (the determiner ``a'' and the adjective
``black''). The sentence ``My friend wrote a book last year'' has a head (the verb ``wrote'') and three dependents (in a constituency sense, the dependents are ``My friend'' as subject, ``a book'' as direct object and ``last years'' as adjunct). The phrase ``a black cat'' is harmonic because all dependents precede the noun. The sentence ``My friend wrote a book last year'' is disharmonic because the verb is surrounded by its dependents. We will refer to the head or the dependents as elements.\footnote{We use the term element to remain agnostic on the syntactic framework, e.g., constituency-based versus dependency-based. } Thus our sequences are formed by $n + 1$ elements. Suppose that the elements are produced sequentially. We aim to investigate if there are head placements that maximize the predictability of elements that have not yet appeared yet. 

Consider that $k$ elements have already been produced. Here we will approach the problem of maximizing the predictability of elements that have not appeared.
We will consider two kinds of prediction problems:
\begin{itemize}
\item
{\em Predictability of the remainder.}\\ 
Imagine that $k$ elements have been produced 
and the predictability of all the remaining $n+1-k$ elements must be maximized. 
We will focus on two cases: $k=1$ (just one
element has already been produced) and $k=n$ (just one element has still
to be produced). 
\item
{\em Predictability of one of the pending elements.}\\
Imagine that $k$ elements have been produced 
and the predictability of one of the pending elements must be maximized individually. 
When $k = n$, predicting one of the pending elements is equivalent to predicting the remainder of the sequence.
This scenario (for an arbitrary $k$) is specially suitable for the {\em in silico} experiments where the next element of a sequence has to be guessed from the elements that have already been produced \citep{Reali2009a}. The power and the interest of that scenario has increased significantly after the success of large language models (LLMs), that are trained to predict the next word or the next element in a broader sense, e.g., the next sentence (see \citet{Yu2024a} and references therein).
\end{itemize}

\subsection{The fundamental assumption: conditional statistical independence between dependents given a head}
\label{subsec:conditional_independence}

Regardless of whether the goal is predicting the remainder or one of the pending elements, our key assumption is that a head and a dependent statistically depend on each other but dependents are statistically independent given a  head. 
We refer to the latter assumption as conditional independence \citep{Ash1965}. This assumption is totally in line with a core assumption of dependency grammar, namely that the bulk of syntactic interactions in a sentence is captured by head-dependent relationships and thus dependent-dependent relations can be neglected \citep{Melcuk1988}. In other words, conditional independence between the dependents of a head is visually equivalent to not adding syntactic dependency arrows between these dependents in a syntactic dependency structure.
For instance, when the head is transitive verb, it means that this verb specifies which words can play the role of subject or object but the subject and object are statistically independent for this verb. Put differently, we assume that head-modifier (or head-complement) relationships are restrictive but dependent-dependent relationships mediated by
heads are flexible. The conditional independence assumption implies abstracting away from the context a sentence is produced. For instance, in the sentence ``John eats apples'', conditional independence implies that given the verbal head ``eats'', the two dependents, ``John'' and ``apples'' are statistically independent. However, if a particular context was considered, it may turn out that liking apples is grammatically possible but unlikely to be produced for the particular person represented by ``John'' because he does not like apples. The conditional independence assumption can be seen as equivalent to assuming that heads and dependents are chosen according to high level abstractions of a language such as grammars and dictionaries, that abstract away to a large extent from the context of use. Therefore, the conditional independence can also be seen as a specific mathematical formalization of the open-choice principle in linguistics \citep{Sinclair1991a}. As Sinclair puts it \citep[p. 175]{Sinclair1991a}, {\em ``In many descriptions of language -- grammars and dictionaries specially -- words are treated as independent items of meaning. Each of them represents a separate choice. Collocations, idioms and other exceptions to that principle are given lower status in the descriptions''.}

\subsection{Predictability as mutual information}

The concept of predictability can be easily defined precisely using standard information theory. We regard $I(X; Y)$, the mutual information between $X$ and $Y$, as a measure of the predictability of $X$ or $Y$. Our view of mutual information as a predictability measure is a restatement of the standard interpretation of $I(X; Y)$ as a reduction in the uncertainty of $X$ due to the knowledge about $Y$ or vice versa, the reduction of the uncertainty of $Y$ due to knowledge about $X$ \citep{Cover2006a}. In plain words, $I(X; Y)$ measures to what extent the values of $X$ behave like identifiers (IDs) of values of $Y$ or vice versa, to what extent the values of $Y$ behave like IDs of values of $X$. Mutual information is a measure of closeness to a one-to-one mapping between the values of $X$ and the values of $Y$. 

We will apply mutual information to measure the predictability of elements of the sequence that have not yet appeared once $k$ elements have already been placed. 
The trick is simple. We will adopt the convention that $X$ stands for the $k$ elements that have appeared and $Y$ stands for elements that have not appeared yet.  
Although mutual information is symmetric ($I(X; Y) = I(Y; X)$), we will adopt the convention that, when we write $I(X, Y)$, the random variable $X$ corresponds to the elements that have already appeared and the random variable $Y$ corresponds to elements to appear. Then $(X,Y)$ is the random variable for the whole sequence. We define $L$ as the random variable for the head. We define $M_i$ as the random variable for the $i$-th dependent appearing in the sequence, with $1 \leq i \leq n$. We use $M_i^j$ to refer to the sequence of random variables $M_i, M_{i+1},...M_{j-1}, M_j$. 

To investigate the problem of the predictability of the remainder when $k$ elements have been produced, we will proceed as follows. If the head has not appeared yet, $X = M_1^k$ and 
$Y = L, M_{k+1}^{n}$. In our mathematical arguments, where exactly the head has appeared among the $k$ first elements or where the head will appear among the $n + 1 - k$ elements to come is irrelevant. Then $Y = L, M_{k+1}^{n} = M_{k+1}^{n}, L$. 
If the head has already appeared, $X = L, M_1^{k-1}$ and $Y = M_{k}^n$. Again  $X = L, M_1^{k-1} = M_1^{k-1}, L$ since where exactly the head has appeared is irrelevant for our arguments.  
As a result of the conventions above, $I(M_1^n; L)$ is the mutual information between (a) $n$ dependents that have already been produced and (b) the head to appear; 
$I(L,M_1^{n-1}; M_n)$ is the mutual information between (a) the elements that have already appeared, i.e. the head and $n-1$ dependents, and (b) the last dependent to appear.

For the problem of the predictability of a pending element when $k$ elements have already been produced, $X$ has the same possible definitions as in the predictability of the remainder problem 
while $Y$ varies. If the head has not appeared yet, $Y = M_j$ with $k + 1 \leq j \leq n$ or $Y = L$.    
If the head has already appeared, $Y = M_j$ with $k \leq j \leq n$.

\subsection{A quick overview of the relevant information theory. }
\label{subsec:quick_overview}

This subsection is just intended to help readers to understand the mathematical proofs that will follow if they are not familiar with information theory. Readers can skip this subsection if they do not want to follow these proofs. 

Consider three random variables $X$, $Y$ and $Z$. 
Mutual information can be conditional. $I(X; Y | Z)$ is the conditional mutual information between $X$ and $Y$ given $Z$.
Mutual information satisfies various properties that are useful for the present article \citep{Cover2006a}: 
\begin{enumerate}
\item
Non-negativity. $I(X; Y) \geq 0$ with equality if and only if $X$ and $Y$ are statistically independent. 
\item
Symmetry. $I(X; Y) = I(Y; X)$.
\item
The chain rule of information, namely
if $X$ decomposes into $X_1$ and $X_2$, namely $X = X_1, X_2$, then 
\begin{eqnarray*}
I(X_1,X_2; Y) & = & I(X_1; Y) + I(X_2; Y | X_1) \\
              & = & I(X_2; Y) + I(X_1; Y | X_2)
\end{eqnarray*}
where $I(X_2; Y | X_1)$ and $I(X_1; Y | X_2)$ are conditional mutual informations. 
\end{enumerate}
We also need to remind properties involving Markov chains. Suppose that random variables $X$, $Y$ and $Z$ form a Markov chain in that order, a property we write as $X \rightarrow Y \rightarrow Z$. 
Then $X$, $Y$ and $Z$ satisfy the following mathematical properties \citep{Cover2006a}
\begin{enumerate}
\item
Conditional independence. $X \rightarrow Y \rightarrow Z$ is equivalent to $X$ and $Z$ being conditionally independent given $Y$.
\item
$I(X; Z | Y) = 0$ thanks to the conditional independence between $X$ and $Y$.  
\item
$Z \rightarrow Y \rightarrow X$ also forms a Markov chain. 
\item
Data processing inequality.  
\begin{equation*}
I(X; Y) \geq I(X; Z)
\end{equation*}
with equality if and only if $I(X; Y| Z) = 0$ (i.e. $X \rightarrow Z \rightarrow Y$ forms a Markov chain). 
\begin{equation*}
I(Y; Z) \geq I(X; Z)
\end{equation*}
with equality if and only if $I(Y; Z| X) = 0$ (i.e. $Y \rightarrow X \rightarrow Z$ forms a Markov chain).
\end{enumerate}

\subsection{Conditional statistical independence} 

The following definition presents the core assumption of the theoretical results on predictability that will be introduced progressively.
Let $l$ be a value that $L$ can take and $m_i$ be a value that $M_i$ can take. Accordingly, $m_i^j = m_i, m_{i+1},...,m_{j-1}, m_j$ is a value that $M_i^j$ can take. 

\begin{definition}[Conditional statistical independence of dependents given the head]
\label{conditional_independence_definition}

Let $p(m_i^j | l)$ be the conditional probability of $M_i^j = m_i^j$ given $L=l$. 
Let $i, j, i', j'$ be integer numbers such $[i,j],[i',j'] \subseteq [1, n]$ and $[i,j] \cap [i',j'] = \emptyset$.  
Then $M_i^j$ and $M_{i'}^{j'}$ are conditionally independent given $L$ iff 
\begin{equation}
p(m_i^j, m_{i'}^{j'} | l) = p(m_i^j| l) p(m_{i'}^{j'}|l).
\label{conditional_independence_equation}
\end{equation} 
\end{definition}

The definition of conditional statistical independence above has various implications.

\begin{property}
Definition \ref{conditional_independence_equation} implies pairwise conditional independence of dependents given a head, i.e.
\begin{equation}
p(m_i, m_j | l) = p(m_i | l ) p(m_j | l)
\label{eq:pairwise_conditional_independence}
\end{equation} 
for $i,j \in [1, n]$ and $i\neq j$.
That definition also implies
\begin{equation}
p(m_i^j | l)= \prod_{h=i}^j p(m_h | l)
\label{eq:conditional_decomposition}
\end{equation}
for $[i,j] \subseteq [1, n]$.
\end{property}
\begin{proof}
Eq. \ref{eq:pairwise_conditional_independence} follows by setting $i=j$ and $i'=j'$ in Eq. \ref{conditional_independence_equation}.  
Eq. \ref{eq:conditional_decomposition} follows by gradually decomposing $m_i^j$ and applying Eq. \ref{conditional_independence_equation} each time. That is
\begin{eqnarray*}
p(m_i^j | l) & = & p(m_i^{j-1}, m_j | l) = p(m_i^{j-1} | l) p(m_j | l) \\
             & = & p(m_i^{j-2}, m_{j-1} | l) p(m_j | l) = p(m_i^{j-2} | l) p(m_{j-1} | l) p(m_j | l) \\
             & = & ... \\
             & = & \prod_{h=i}^j p(m_h | l).
\end{eqnarray*} 
\end{proof}
 
\begin{lemma}[Markov chain]
\label{Markov_chain_lemma}
Under conditional statistical independence of dependents given a head (Definition \ref{conditional_independence_definition}) and
for integers $i, j, i', j'$ such that $1 \leq i < j \leq n$ and $1 \leq i' < j' \leq n$ and $[i,j] \cap [i',j'] = \emptyset$,  
we have that 
\begin{enumerate}
\item
$M_i^j \rightarrow L \rightarrow M_{i'}^{j'}$ and its reverse, i.e. $M_{i'}^{j'} \rightarrow L \rightarrow M_i^j$, form a Markov chain.
\item
\begin{equation*} 
I(M_i^j; M_{i'}^{j'} | L) = 0. 
\end{equation*}
\end{enumerate}
\end{lemma}
\begin{proof}
We prove first the 1st part. Without loss of generality, suppose that $j < i'$. If $M_i^j$, $L$ and $M_{i'}^{j'}$ in that order formed a Markov chain, i. e. $M_i^j \rightarrow L \rightarrow M_{i'}^{j'}$, then  
\begin{equation*}
p(m_i^j, m_i^j, l) = p(m_i^j) p(l | m_i^j) p(m_{i'}^{j'} | l).
\end{equation*}
Such a definition of Markov chain follows from conditional independence, that is 
\begin{eqnarray*}
p(m_i^j, m_{i'}^{j'} | l) & = & p(m_i^j| l) p(m_{i'}^{j'} | l) \quad \quad \text{conditional independence} \\
                          & = & \frac{p(m_i^j, l)}{p(l)} p(m_{i'}^{j'} | l) \quad \quad \text{definition of conditional probability} \\
                          & = & \frac{p(m_i^j) p(l | m_i^j) p(m_{i'}^{j'} | l)}{p(l)} \quad \quad \text{definition of conditional probability} \\                          
                          & = & \frac{p(m_i^j, l, m_{i'}^{j'})}{p(l)} \quad \quad \text{definition of Markov chain}.                      
\end{eqnarray*}
By well-known properties of Markov chains, the reverse of $M_i^j \rightarrow L \rightarrow M_{i'}^{j'}$, namely 
$M_{i'}^{j'} \rightarrow L \rightarrow M_i^j$, also forms a Markov chain (Section \ref{subsec:quick_overview}).

The second part follows trivially because mutual information is zero in case of statistical independence (Section \ref{subsec:quick_overview}). 
\end{proof}

\subsection{Predictability of the remainder of the sequence}

The following theorem indicates the optimal solutions for the predictability of the remainder problem.  

\begin{theorem}[Predictability of the remainder of the sequence]
\label{predictability_theorem}
Assume that the whole sequence is formed by one head and $n\geq 1$ dependents and also that dependents are conditionally independent given a head according to Definition \ref{conditional_independence_definition}.
When $k=1$,
\begin{equation}
I(L; M_1^n) \geq I(M_1; L, M_2^n)
\label{inequality_to_prove1}
\end{equation}
with equality iff 
\begin{itemize}
\item
$n = 1$ or
\item
$n \geq 2$ and $L \rightarrow M_1 \rightarrow M_2^n$ forms a Markov chain.
\end{itemize}
When $k = n$, 
\begin{equation}
I(M_1^{n}; L) \geq I(L, M_1^{n-1}; M_n)
\label{inequality_to_prove2}
\end{equation}
with equality iff 
\begin{itemize}
\item
$n = 1$ or
\item
$n \geq 2$ and $M_1^{n-1} \rightarrow M_n \rightarrow L$ forms a Markov chain.
\end{itemize}
\end{theorem}
\begin{proof}
We begin with the proof of Eq. \ref{inequality_to_prove1} (the case $k=1$). 
When $n = 1$, we have equality due to the symmetry property of mutual information. 
When $n > 1$, Eq. \ref{inequality_to_prove1} can be written equivalently as
\begin{equation*}
I(L; M_1, M_2^n) \geq I(M_1; L, M_2^n).
\end{equation*}
By the chain rule of mutual information and the symmetry of mutual information, the previous inequality becomes
\begin{eqnarray*} 
I(L; M_1) + I(L; M_2^n | M_1) \geq I(L; M_1) + I(M_1; M_2^n | L)
\end{eqnarray*}
and then
\begin{eqnarray}
I(L; M_2^n | M_1) \geq I(M_1; M_2^n | L).
\label{inequality_to_proof_after_chain_rule_equation}
\end{eqnarray}
Under conditional statistical independence of dependents (Definition \ref{conditional_independence_definition}),
Lemma \ref{Markov_chain_lemma} indicates that  
$M_1 \rightarrow L \rightarrow M_2^n$ forms a Markov chain and 
\begin{equation*} 
I(M_1; M_2^n | L) = 0.
\end{equation*}
Then Eq. \ref{inequality_to_proof_after_chain_rule_equation} finally becomes
\begin{equation*}
I(L; M_2^n | M_1) \geq 0
\end{equation*}
with equality if and only if $L \rightarrow M_1 \rightarrow M_2^n$ forms a Markov chain (Lemma \ref{Markov_chain_lemma}).

The proof of Eq. \ref{inequality_to_prove2} follows by symmetry from the proof of Eq. \ref{inequality_to_prove1}. 
\end{proof}

For the linguistically or cognitively oriented reader, the preceding theorem is best understood by its implications for incremental processing expressed in the following corollary. 

\begin{corollary}[The optimality of word order harmony]
\label{predictability_corollary}
Theorem \ref{predictability_theorem} states that, when 
\begin{itemize}
\item
there are at least two dependents ($n\geq 2$) and 
\item
the predictability of the pending elements is to be maximized under the conditional independence assumption, 
\end{itemize}
then the optimal solutions are
\begin{itemize} 
\item
Head first when just a single element has been produced ($k=1$). 
\item
Head last when exactly $n$ elements have been produced ($k=n$).  
\end{itemize}
To sum up, a harmonic placement of the head is optimal with respect to maximum predictability but for distinct reasons depending on the placement of the head (first or last).
\end{corollary}

\subsection{Predictability of a pending element}

So far, we have approached the problem of the predictability of the remainder of the sequence. Now, our attention turns to the prediction of each of pending elements of the sequence when $k$ elements have already been produced. The next theorem indicates the optimal configurations for the problem of maximizing the predictability of just one of the pending elements.  

\begin{theorem}[Predictability of a pending element]
\label{predictability_of_the_next_theorem}
Assume that 
\begin{itemize}
\item
The whole sequence is formed by a head and $n$ dependents.
\item
$k$ elements have already been produced, with $1 \leq k \leq n$.
\item
Dependents are conditionally independent given a head.
\end{itemize}
Then  
\begin{itemize} 
\item
{\em Predictability of the head.} A dependent cannot be more predictable than the head, i.e.  
\begin{enumerate}
\item
If the head has been produced and $k \leq j \leq n$, 
\begin{equation}
\label{prediction_of_the_next_element1_equation}
I(L, M_1^{k - 1}; M_j) \leq I(M_1^k; L)  
\end{equation}
with equality iff 
  \begin{itemize}
  \item
  $k=1$ or
  \item
  $k\geq 2$ and $L \rightarrow M_j \rightarrow M_1^{k-1}$ forms a Markov chain.
  \end{itemize} 
\item
If the head has not been produced and $k < j \leq n$,
\begin{equation}
\label{prediction_of_the_next_element2_equation}
I(M_1^k; M_j) \leq I(M_1^k ; L)
\end{equation}
with equality iff $M_1^k \rightarrow M_j \rightarrow L$ forms a Markov chain (Lemma \ref{Markov_chain_lemma}).
\end{enumerate}
\item
{\em Predictability of a dependent.} A dependent cannot be less predictable if the head has already been produced, i.e.
for any integer $j$ such that $k < j \leq n$
\begin{enumerate}
\item[(3)] 
  \begin{equation}
  I(L, M_1^{k-1}; M_j) \geq I(M_1^k; M_j) 
  \label{prediction_of_the_next_element3_equation}
  \end{equation}
  with equality iff $L \rightarrow M_1^k \rightarrow M_j$ forms a Markov chain.
\end{enumerate}
\end{itemize}
\end{theorem}
\begin{proof}
The first part (Eq. \ref{prediction_of_the_next_element1_equation}) follows from the case $k=n$ of Theorem \ref{predictability_theorem} replacing $M_1^{n-1}$ by $M_1^k$. The second part (Eq. \ref{prediction_of_the_next_element2_equation}) follows from the data processing inequality noting that $M_1^k \rightarrow L \rightarrow M_j$ forms a Markov chain. As for the third part (Eq. \ref{prediction_of_the_next_element3_equation}), notice that 
\begin{eqnarray*}
I(L, M_1^{k-1}; M_j) & =    & I(L; M_j) + I(M_1^{k-1}; M_j | L) \quad \quad \text{chain rule of mutual information}\\
                     & =    & I(L; M_j) \quad \quad \text{Lemma \ref{Markov_chain_lemma}} \\
                     & \geq & I(M_1^k; M_j) \quad \quad \text{data processing inequality with } M_1^k \rightarrow L \rightarrow M_j
\end{eqnarray*}
with equality iff $L \rightarrow M_1^k \rightarrow M_j$ forms a Markov chain.
\end{proof}

The following corollary distils the take home message of the previous theorem.

\begin{corollary}[Predictability of a pending element]
\label{predictability_of_the_next_corollary}
Theorem \ref{predictability_of_the_next_theorem} states that, when 
\begin{itemize}
\item
at least two elements have been produced ($k\geq 2$)
\item
the predictability of the pending elements is to be maximized under the conditional independence assumption,
\end{itemize}
then
\begin{itemize} 
\item
{\em Predictability of the head}. In general, heads are more predictable than dependents. Technically, the head is at least as predictable as a dependent. 
\item
{\em Predictability of a dependent}. In general, dependents are more predictable when a head has been produced. 
Technically, a dependent cannot be less predictable if the head has already been produced. 
\end{itemize}
\end{corollary}

\subsection{The asymmetry between heads and dependents}

The third part of Theorem \ref{predictability_of_the_next_theorem} (Eq. \ref{prediction_of_the_next_element3_equation}) shows that producing the head earlier is optimal in terms of maximizing the predictability of dependents and thus head first orders are eventually optimal from the point of view of the predictability of each of the pending dependents. However, the virtue of bringing the head forward is shadowed by the fact that a dependent cannot be more predictable than a head (Theorem \ref{predictability_of_the_next_theorem}) and the fact that once the head has been placed, adding more dependents will not increase the predictability of each of the pending dependents as it will be shown next. Indeed the predictability of one of the pending dependents does not depend on previous dependents once the head is known as the following lemma shows. 

\begin{lemma}[The irrelevance of dependents]
\label{irrelevance_of_modifiers_lemma}
Suppose that
\begin{itemize} 
\item
The whole sequence is formed by a head and $n$ dependents. 
\item
The head and $k$ dependents have actually been produced.
\item
Dependents are conditionally independent given the head.
\end{itemize}
Then, the dependents that have already been produced are irrelevant for predicting each of the pending dependents, i.e.
for integers $k$ and $j$ such that $1 \leq k < j \leq n$,   
\begin{equation}
I(L, M_1^k; M_j) = I(L; M_j).
\label{irrelevance_of_modifiers_equation}
\end{equation}
\end{lemma}
\begin{proof}
We have 
\begin{eqnarray*}
I(L, M_1^k; M_j) & = & I(L; M_j) + I(M_j; M_1^k | L) \quad \quad \text{chain rule of mutual information}\\
                 & = & I(L; M_j) \quad \quad \text{Lemma \ref{Markov_chain_lemma}}
\end{eqnarray*}
\end{proof}

Next we demonstrate that postponing the head is optimal in terms of maximizing its predictability and thus head last orders are eventually optimal from the point of view of the predictability of the head. The power of the argument stems from combining (a) the fact that a dependent cannot be more predictable than a head even when the head has actually been produced (Theorem \ref{predictability_of_the_next_theorem}) with (b) the fact that adding dependents will ease, in general, the prediction of the head, while this will not happen to dependents if the head has already been produced (Lemma \ref{irrelevance_of_modifiers_lemma}). The following theorem seeks to illuminate converging evidence suggesting a preference for head last orderings under certain conditions \citep{Goldin-Meadow2008a,Langus2010a,Meir2010b,Hall2013a,Hall2014a,Schowstra2014a,Futrell2015a,Reali2009a}.
\begin{theorem}[The virtue of postponing the head]
\label{optimality_of_postponing_the_head_theorem}
Assume that 
\begin{itemize}
\item
The whole sequence is formed by a head and $n$ dependents. 
\item
$k$ elements have already been produced with $1 \leq k < n$.
\item
Dependents are conditionally independent given a head (Definition \ref{conditional_independence_definition}).
\end{itemize}
Then we have that
{\small
\begin{equation*}
\begin{array}{ccc}
I( M_1^k ; L) & \leq_{(1)} & I(M_1^{k+1}; L ) \\
                                       &      &       \\ 
             \leqcup_{(2)}             &      &       \leqcup_{(3)}             \\  
                                       &      &       \\ 
I( L, M_1^{k-1}; M_k) & =_{(4)} & I(L, M_1^k; M_{k+1} ) \\
                                       &      &       \\ 
             \leqcup_{(5)}             &      &       \leqcup_{(6)}             \\ 
                                       &      &       \\ 
I( M_1^k ; M_{k+1} )          & \leq_{(7)} & I( M_1^{k+1}; M_{k+2} ) \\
                                       &     & \mbox{(for~} k < n - 1\mbox{)} \\ 
\end{array}
\end{equation*}
}
\end{theorem}
\begin{proof}
The $\leqcup$ relationships (inequalities (2), (3), (5) and (6)) follow from Theorem \ref{predictability_of_the_next_theorem}. 
The top $\leq$ relationship (inequality (1)) results from  
\begin{eqnarray*}
I( M_1^k, M_{k+1}; L ) & =    & I(M_1^k; L ) + I(M_{k+1}; L | M_1^k) \quad \quad \text{chain rule of mutual information}\\
                      & \geq & I(M_1^k; L ) \quad \quad \text{non-negativity of mutual information} 
\end{eqnarray*}
with equality iff $I(M_{k+1}; L | M_1^k) = 0$,
i.e. $M_{k+1} \rightarrow M_1^k \rightarrow L$ forms a Markov chain (Lemma \ref{Markov_chain_lemma}).
The central equality (inequality (4)) is proven by Lemma \ref{irrelevance_of_modifiers_lemma} (applying Eq. \ref{irrelevance_of_modifiers_equation}).
The bottom $\leq$ relationship (inequality (7)) results from 
\begin{eqnarray*}
I( M_1^k, M_{k+2}; M_{k+1} ) & =    & I(M_1^k; M_{k+2} ) + I( M_{k+1} ; M_{k+2} | M_1^k) \quad \quad \text{chain rule of mutual information} \\
                             & \geq & I(M_1^k; M_1^{k+1} ) \quad \quad \text{non-negativity of mutual information} 
\end{eqnarray*}
with equality iff $I(M_{k+1}; M_{k+2} | M_1^k) = 0$,
i.e. $M_{k+1} \rightarrow M_1^k \rightarrow M_{k+2}$ forms a Markov chain (Lemma \ref{Markov_chain_lemma}).
\end{proof}

The following corollary distills the take home message of the previous theorem.

\begin{corollary}[The virtue of postponing the head]
\label{optimality_of_postponing_the_head_corollary}
Theorem \ref{optimality_of_postponing_the_head_theorem} shows a diagram of all possible strategies in terms of which elements have been produced and the next element to predict.
\begin{itemize}
\item
The right-upper corner shows that postponing the head is the globally optimal strategy in terms of predictability maximization.
\item
The left column  of inequalities (inequalities (2) and (5)) show that the best strategy in terms of predictability maximization is to predict the head, the worst strategy is to predict a dependent when the head has not been produced yet; the intermediate strategy is to predict a dependent once the head has been produced. 
\item
The middle-column inequalities ((1), (4) and (7)) indicate that producing an element can increase predictability except when the head has already been produced, where predictability remains constant.
\item
The right-column inequalities ((3) and (6)) simply update the left-column inequalities by adding one more element to the sequence. 
\end{itemize}
\end{corollary}

\section{Discussion}

\label{discussion_section}

Next we examine the implications of the above theoretical results for concrete word order problems, with emphasis on the ordering of object, subject and verb. Section \ref{maximum_predictability_disussion_section} examines phenomena that can be illuminated mainly by the action of predictability maximization and the mathematical results in Section \ref{predictability_section}. Second \ref{general_discussion_section} reviews the limits of that approach and connects it with a general theory of word order. 

\subsection{How predictability maximization shapes word order}
\label{maximum_predictability_disussion_section}

\subsubsection{The frequency of the possible order of subject, object and verb in languages} 

The principle of predictability maximization predicts a harmonic word order (Corollary \ref{predictability_corollary}). However, the optimal placement of the head depends on the target of predictability, that can be either the verb or its dependents. 
Placing the verb at the end is optimal in terms of the predictability of the verb. This is the case of 70.6\% of linguistic families and 44.7-47.4\% of languages (Table \ref{word_order_statistics_table}). Placing the verb first is optimal in terms of the predictability of both the object and the subject. This is only the case of a small fraction of families or languages showing a dominant word order (12.4\% of families and 10.5-13.2\% of languages according to Table \ref{word_order_statistics_table}). 

Our mathematical results and Table \ref{word_order_statistics_table} suggests three possibilities: 
\begin{itemize}
\item
{\em Explanation 1}. A free parameter explanation, where languages adopt the target of prediction arbitrarily (then most languages would have adopted the verb).
\item
{\em Explanation 2}. A prior cognitive preference for the verb as target (which would explain why verb last is more likely than verb first). That would not explain, however, the unexpected high frequency of verb medial languages, which are disharmonic (Table \ref{word_order_statistics_table}). 
\item
{\em Explanation 3}. Some natural advantage in postponing the head that may explain the unexpected high frequency of verb medial orders. That would explain why the number languages/families increases as the verb moves to the end (Table \ref{word_order_statistics_table}). Here we use the term natural in two senses: supported by mathematical truths (theorems) or by knowledge on how cognition works. 
\end{itemize}
Explanations 2 and 3 assume a natural asymmetry between heads and dependents. Next we will make the case for explanations 2 or 3. 

As for the mathematical asymmetry, we have seen that verbs need to be postponed to maximize their predictability but must be brought forward to maximize the predictability of dependents (Corollary \ref{predictability_of_the_next_corollary}). However, 
we have already seen that heads and dependents are not symmetric with respect to predictability. The strategy of maximizing the predictability of the verb has various advantages over that of maximizing the predictability of its dependents (Corollary \ref{optimality_of_postponing_the_head_corollary}):
\begin{itemize}
\item 
The verb is in general more predictable than its dependents (regardless of whether the verb has already been produced or not).
\item 
While the predictability of the head is expected to increase as the verb is postponed, the predictability of its pending dependents remains constant once the verb has already been produced. Once the verb has been produced the dependents that are produced do not help to predict each of the pending dependents.
\item
If the verb has not been produced yet, none of the pending dependents will be more predictable than the verb.
\end{itemize}

From the standpoint of how cognition works, support for the asymmetry between verbs and their dependents, which have typically a noun as head, comes from child language research \citep{Saxton2010a_Chapter6}. For children, nouns are easier to learn than verbs (e.g., \citet{Imai2008a}) and actions (typically represented by verbs) are harder to pick up, encode and recall than objects (typically represented by nouns) (e.g., \citet{Gentner1982a,Gentner2006a,Imai2005a}). Verb meanings are more difficult to extend than those of nouns (e.g., \citet{Imai2005a}). See \citet{McDonough2011a} for an overview of hypotheses on the hardness of verbs with regard to nouns. Furthermore, arguments for the greater difficulty of verbs for infants can easily be
generalized to adults beyond the domain of learning. 
For these reasons, a communication system that aims at facilitating the learning of the most difficult items, i.e. verbs, by infants, may favour the strategy of maximizing the predictability of the verb (leading to verb last) over the strategy of maximizing the predictability of the nouns (leading to head first). Besides, computer and eye-tracking experiments indicate that the arguments that precede the verb help to predict it \citep{Konieczny2003a}.

\subsubsection{Verb last in computer prediction experiments}

Computer simulation is a powerful tool for word order research \citep{Reali2009a,Gong2009a,Konieczny2003a}.
SOV has been obtained in two-stage computer simulation experiments with recurrent neural networks (learners) that have addressed the problem of the emergence of word order from the point of view of sequential learning \citep{Reali2009a}. During the first stage, networks learned to predict the next element of number sequences and the best learners were selected. During the second stage, language was introduced and coevolved with the learners. The best language learners and the best learned language were selected. The best language learners had to comply with the additional constraint of maintaining the performance on number prediction of the first stage. 
Notice that predictability is an explicit selective pressure for the neural networks in both stages and that the languages that were selected in the second stage must have been strongly influenced by the pressure to maximize predictability. This suggests that SOV surfaces in these experiments because of the virtue of maximizing the predictability of the head (Corollary \ref{optimality_of_postponing_the_head_corollary}).

\subsubsection{The preference for verb/action last in simple sequences}

Gestural communication experiments with only one head, i.e. a verb or an action, and two dependents, i.e. a subject or actor and an object or patient, show a preference for placing the
head at the end in simple sequences \citep{Goldin-Meadow2008a,Langus2010a,Meir2010b,Hall2013a,Hall2014a,Schowstra2014a,Futrell2015a}. Our results above indicate that this ordering is optimal from the point of view of predictability of the verb/action and suggests that signers may have adopted the strategy of maximizing the predictability of actions (both as producers and receivers) because its advantage over that of maximizing the predictability of the subject/actor or object/patient or because the increasing gain of postponing the head (Corollary \ref{optimality_of_postponing_the_head_corollary}).

\subsubsection{The evolution of the order of subject, object and verb}

SOV might be an early step in the evolution of spoken languages \citep{Gell-Mann2011a, Newmeyer2000} and is often retrieved in emerging sign languages \citep{Meir2010a, Sandler2005a}. For a discussion of how this might result from predictability maximization and other questions relating to the evolution of word order, we refer the reader to previous work \iftoggle{anonymous}{\citep{Ferrer2014a_fake,Ferrer2013f_fake}.}{\citep{Ferrer2014a,Ferrer2013f}.}

\subsubsection{The preference for noun first or last in simple sequences} 

The noun tends to be placed at one of the ends of a noun phrase formed by noun, adjective, numeral and demonstrative \iftoggle{anonymous}{\citep{Ferrer2023b_fake}}{\citep{Ferrer2023b}} as expected by maximum predictability. 
Consistently, experiments on learning an artificial language involving a head, i.e. a noun, a few other dependents (e.g., a numeral and an adjective) have demonstrated a robust preference for noun first or last  \citep{Culbertson2017a,Martin2019a,Culbertson2020a,Martin2024a}.
Furthermore, our mathematical arguments (Corollary \ref{optimality_of_postponing_the_head_corollary}) predict a theoretical advantage for verb last that is consistent with the hypothesis of a universal underlying order in which all modifiers precede the noun \citep{Cinque2005a}. 
However, the frequencies of the orders do not show an asymmetry for head last as it happens with the order of the subject, object and verb and as expected by our mathematical framework. Indeed, noun first orders are much more frequent than noun last orders \citep{Dryer2018a}.
We believe that various factors may explain these inconsistencies. First, the object and verb are on top of the clause hierarchy. The order of elements within the noun phrase, may be determined by the order of higher level components. Then the assumption of our framework, namely a single explicit head, may fail to provide a sufficient approximation. Indeed, the relative order of adjectives with respect to nouns could be determined by cognitive or evolutionary processes involving the order of subject, verb and object \iftoggle{anonymous}{\citep{Ferrer2013e_fake}.}{\citep{Ferrer2013e}.} 
Second, such a noun-phrase consists of 5 words that are easier to fit into the span of short term memory (STM) without compression \citep{Mathy2012a} than a SOV structure, as  
subjects, objects and verbs are likely to be formed by more than one word in spoken languages.
Third, the benefits of the {\em a priori} advantage of postponing the head (corollary \ref{optimality_of_postponing_the_head_corollary}) may depend on the time scale. That noun phrase is made of individual spoken words that are likely to be short \iftoggle{anonymous}{\citep{Ferrer2023b_fake}}{\citep{Ferrer2023b}} and thus operates on small time scales. In contrast, subjects, objects and verbs are likely to be formed by more than one word in spoken languages. Gestures take more time to produce than words \citep{Wilson2006a} and 
the verbal working memory capacity for sign languages appears to be smaller than the short-term memory capacity for spoken language \citep{Malaia2019a}. Therefore, subjects objects and verbs in spoken languages or even short sequences of gestures corresponding to agent-patient-action may operate on longer scales where the theoretical advantage of postponing the head can be exploited. 
   
\subsection{General word order theory}

\label{general_discussion_section}

We review the limitations of the research above and how they can be overcome by a broader theory of word order. 

\subsubsection{Other word order principles}

Here we have investigated the optimal placement of the head and its dependents from the perspective of the principle of predictability maximization, which predicts a harmonic order (either head first or last). However, the optimal placement of the head according to the principle of syntactic dependency minimization predicts a disharmonic order, namely the head should be placed in the middle of the sequence of dependents 
\iftoggle{anonymous}{\citep{Ferrer2008e_fake,Ferrer2013e_fake}.}{\citep{Ferrer2008e,Ferrer2013e}.}
That could explain the unexpected high frequency verb medial orders compared to verb initial orders in Table \ref{word_order_statistics_table}.
Therefore, there are at least two possible explanations for the increase in the number of languages/families as the verb moves towards the end: (a) Explanation 3 above (a gradual advantage of postponing the verb caused by predictability maximization alone) and (b) Explanation 2 above (maximization of the predictability of the head) in combination with interference from syntactic dependency minimization.
Besides, the principle of swap distance minimization is able to explain cognitive effort in lab experiments more accurately than a preference for verb last in certain SOV languages \iftoggle{anonymous}{\citep{Ferrer2023a_fake}.}{\citep{Ferrer2023a}.}

\subsubsection{The conflict between word order principles} 

Predictability maximization or its sister principle, surprisal minimization \iftoggle{anonymous}{\citep{Ferrer2013f_fake}}{\citep{Ferrer2013f}} and syntactic dependency distance minimization make incompatible predictions on the optimal placement of the head. In single-head structures, the latter predicts that the head should be put at the center of the sequence \iftoggle{anonymous}{\citep{Ferrer2013e_fake}}{\citep{Ferrer2013e}} while predictability maximization/surprisal minimization predict a harmonic placement as it has been shown here and in previous research \iftoggle{anonymous}{\citep{Ferrer2013f_fake}.}{\citep{Ferrer2013f}.} A conflict between the principle of predictability maximization (or its sister, surprisal minimization) and the principle of syntactic dependency distance minimization has been demonstrated mathematically \iftoggle{anonymous}{\citep{Ferrer2013f_fake}.}{\citep{Ferrer2013f}.} This is a conflict between short-term memory (syntactic dependency minimization) and long-term memory (that provides the statistical information from the past to guess which next elements are more likely given the current context).  
Versions of this conflict between memory and surprisal/ predictability have been presented under name of memory-surprisal trade-offs \citep{Hahn2020a,Hahn2021a}. 

\subsubsection{When is predictability maximization more likely to manifest} 

It has been argued theoretically and confirmed empirically that harmonic orders are more likely when dependency distance minimization (short term memory) is less likely to interfere, i.e. when (a) sequences are shorter and (b) words are shorter \iftoggle{anonymous}{\citep{Ferrer2019a_fake,Ferrer2020b_fake,Ferrer2023b_fake}.}{\citep{Ferrer2019a,Ferrer2020b,Ferrer2023b}.}
We believe that this the reason why the preference for head last in gestural experiments with simple sequences \citep{Goldin-Meadow2008a,Langus2010a,Meir2010b,Hall2013a,Hall2014a,Schowstra2014a,Futrell2015a} is lost in complex sequences \citep{Langus2010a}.
We believe that such interaction between forces illuminates the ancient view of SOV as a default basic word order \citep{Givon1979a,Newmeyer2000}: SOV is the default order under specific experimental conditions.

\subsubsection{The low frequency of OSV}
 
  We have presented arguments for a verb last bias, namely SOV or OSV, as a side-effect of predictability maximization. Accordingly,
  SOV is the 1st or 2nd more frequent order but OSV is the 3rd or 4th most frequent order in gestural experiments independently of the experimental conditions but involving a small set of languages \citep{Hall2013a,Futrell2015a}.
On a world-level scale, the fact that OSV is the least frequent dominant word order \citep{wals-84,Hammarstroem2016a}. From a cognitive standpoint, the unexpected low frequency of OSV can be explained by (a) subject first bias or (b) the attraction exerted by dependency distance minimization towards a region of the permutation space closer to SVO and then away from OSV (see the Discussion section of \iftoggle{anonymous}{\citet{Ferrer2023b_fake}}{\citet{Ferrer2023b}} for a summary of the argument). In addition, evolutionary processes could explain why OSV is being abandoned \iftoggle{anonymous}{\citep{Ferrer2013e_fake}.}{\citep{Ferrer2013e}.}   

\subsubsection{The low frequency of OVS}

If the abundance of SVO is explained for being an optimal order with respect to syntactic dependency distance minimization for putting the verb in the middle, why is OVS a rare order, although it also places the verb in the middle? The same arguments for the low frequency of OSV can be applied here.  

\subsubsection{Reductionistic versus non-reductionistic modelling} 

Reductionistic models of word order preferences are models that reduce the word order preferences of the whole to word order preference of components \iftoggle{anonymous}{\citep{Ferrer2008f_fake}.}{\citep{Ferrer2008f}.} For instance, the standard model of typology reduces the preference of the whole formed by subject, object and verb to a preference 
for SV (over VS) and SO (over OS) \citep{Cysouw2008a}, which leads to SOV and SVO as the most preferred orders. \citet{Gibson2013a} reduce the preference of the whole to a SO (over OS) bias and a verb final bias, which leads to SOV as the most preferred order. However, word order is inherently non-reductionistic. Here we have shown that word order preferences emerge only when there are two or more dependents ($n\geq 2$). When $n = 1$, the symmetry of mutual information yields the same predictability for head first and for head last. Besides, the principle of syntactic dependency distance minimization and the principle of swap distance minimization also require $n \geq 2$, for word order preferences to emerge \iftoggle{anonymous}{\citep{Ferrer2013e_fake,Franco2024a_fake}.}{\citep{Ferrer2013e,Franco2024a}.} \footnote{Under the simplifying assumption that dependency distance is measured in words for the principle of syntactic dependency distance minimization.} 

\subsubsection{Languages lacking a dominant order}

To produce Table \ref{word_order_statistics_table}, we neglected the problem of languages lacking a dominant order.
Reductionistic models of word order preferences fail to explain the origins of languages lacking a dominant order as they are designed to predict preferences, not absence of them.
If dominant orders are produced by the action of word order principles, why are there languages showing no dominant order? The absence of a dominant order does not imply absence of the primary cognitive principles discussed above. Instead, the lack of a dominant order would reflect a competition between word orders (recall the conflict above) or the immaturity of the communication system \iftoggle{anonymous}{\citep{Ferrer2014a_fake,Ferrer2013f_fake}.}{\citep{Ferrer2014a,Ferrer2013f}.} 

\subsubsection{Generalizability and cognitive grounding} 

It is obvious that the reductionistic models of word order above do not generalize to syntactic structures that contain elements other than S, O and V. 
In contrast, the scope of our mathematical framework (summarized in Theorem \ref{optimality_of_postponing_the_head_theorem}) goes beyond linguistics as it holds for any sequence formed by an element of kind $A$ (originally a head) and one or more elements of kind $B$ (originally dependents), where the elements of kind $B$ are conditionally independent given an element of kind $A$.

The reductionistic models of word order do not even generalise when restricted to S, O and V. Both reductionistic models fail to explain the most attested pairs of alternating orders in languages that lack a single dominant order. The most attested pairing is SOV-SVO and the standard model of typology \citep{Cysouw2008a} predicts SOV-SVO as the most preferred orders. However, the second most attested pairing is VSO-VOS where VSO violates the SV preference and VOS violates both the SV and the SO preferences. According to that model, one would expect that the 2nd most attested pairing involves one order that does not violate any of the two pairwise preferences. The reductionistic model by \citet{Gibson2013a}, predicts SOV as the most preferred order, but the second and the third most attested pairings (VSO-VOS and SVO-VSO) do not involve SOV. The principle of swap distance minimization accounts for such pairings more accurately \iftoggle{anonymous}{\citep{Ferrer2016c_fake}.}{\citep{Ferrer2016c}.}

A verb final bias, or even a harmonicity bias, and predictability maximization are not epistemically equivalent. A verb final bias, that is justified on frequency counts \citep{Gibson2013a}, overfits the data because it does not generalize to other structures such as the noun phrase \iftoggle{anonymous}{\citep{Ferrer2023b_fake}.}{\citep{Ferrer2023b}.} A verb final bias is a simple induction step. Its generalization to any head and both ends of the sequence, that is, a harmonicity bias \citep{Biberauer2013a,Dryer2018a}, has still a strong inductive component and lacks cognitive grounding. In contrast, predictability maximization is grounded on an elementary skill of the mind: predicting the next element(s) \citep{Kamide2008a,Onnis2022a} and comes with a mathematical theory that anticipates physical reality and generates knowledge without induction by the power of mathematical inference from elementary assumptions (Section \ref{predictability_section}, Theorem \ref{optimality_of_postponing_the_head_theorem}). Exploiting such skill underpins the success of LLMs that emulate human cognitive and linguistic abilities \citep{Yu2024a}. Finally, just as a sanity check, speakers of SOV languages do exhibit a preference for verb last when tested in the lab \iftoggle{anonymous}{\citep{Ferrer2023a_fake}.}{\citep{Ferrer2023a}.} 

\subsection{Final remarks}

We have just provided some elementary mathematical results to understand the implications of the placement of the head for predictability in single head structures. These findings shed light on how the principle of predictability maximization may determine the placement of the head. More realistic scenarios should be investigated. First, our theoretical framework should be extended to the case of multiple head structures such as those of complex sentences. Second, relaxations of the conditional independence assumption should be explored. Finally, thanks to the symmetry property of mutual information, our mathematical framework for the problem of predictability (prediction of next elements) can also be applied to the problems of retrodiction (prediction of previous elements) \citep{Onnis2021a} or integration \citep{Onnis2022a}. 


\iftoggle{anonymous}{}
{ 

\section*{Acknowledgements}

We are very grateful to L. Alemany-Puig and L. Onnis for valuable feedback and spotting many typos. All remaining errors are totally ours. We have also benefited from discussions with G. Fenk-Oczlon and the contents of the talk that she gave at the 16th International Cognitive Linguistics Conference (August 2023), ``Working memory constraints: Implications for efficient coding of messages''. The latter helped us to interpret the findings on gestural experiments versus spoken languages. We owe the observation that the conflict between predictability maximization and syntactic dependency distance minimization is a conflict between short-term memory and long-term memory to G. Heyer and A. Mehler (5 June 2009).
This research is supported by a recognition 2021SGR-Cat (01266 LQMC) from AGAUR (Generalitat de Catalunya) and the grants AGRUPS-2023 and AGRUPS-2024 from Universitat Politècnica de Catalunya.

{\em Historical note.} Early versions of the main of results of this article were presented in a series of lectures. The 1st lecture was ``Memory versus predictability in syntactic dependencies'' for the Kickoff Meeting ``Linguistic Networks'' which was held in Bielefeld University (Germany) on June 5, 2009. 
The 2nd lecture was ``Word order as a constraint satisfaction problem. A mathematical approach'' for the meeting ``Complexity in language: developmental and evolutionary perspectives'', that was held in Lyon (France) on May 23-24, 2011. The 3rd lecture was ``The world war of word orders. Insights from complex systems science'' for the ``Annual Workshop of the Cognitive Science and Language master'', Barcelona (Catalonia) on June 7, 2011. We thank the organizers and the participants of these meetings for valuable discussions. 
The first version of the work, ``Optimal placement of heads: a conflict between predictability and memory'', was submitted for publication for the 1st time in 2009 and was rejected in distinct venues.
Since 2009, at least R. Levy, F. Jaeger, E. Gibson, S. Piantadosi and R. Futrell had access to different versions of the unpublished manuscript. That manuscript circulated within Gibson's lab and inspired at least their research on the order of subject, object and verb \citep{Gibson2011a,Gibson2013a,Futrell2015b} and may have also influenced later work on memory-surprisal trade-offs \citep{Hahn2020a,Hahn2021a}. 
Evolved versions of various components of the unpublished manuscript have already appeared \citep{Ferrer2014a,Ferrer2013e,Ferrer2013f}. This article closes the publication of the bulk of the parts of the early manuscript. This article is dedicated to the memory of late researchers who supported the presentation (A. Costa-Martínez 1970-2018) or publication  (G. Altmann 1931-2020) of that research line.   
}

\bibliographystyle{apacite}

\iftoggle{anonymous}{
\bibliography{fake,../../../../../Dropbox/biblio/rferrericancho,../../../../../Dropbox/biblio/complex,../../../../../Dropbox/biblio/ling,../../../../../Dropbox/biblio/cl,../../../../../Dropbox/biblio/cs,../../../../../Dropbox/biblio/maths}
}
{
\bibliography{../../../../../Dropbox/biblio/rferrericancho,../../../../../Dropbox/biblio/complex,../../../../../Dropbox/biblio/ling,../../../../../Dropbox/biblio/cl,../../../../../Dropbox/biblio/cs,../../../../../Dropbox/biblio/maths}
}

\end{document}